\documentclass{article}

\usepackage{microtype}
\usepackage{graphicx}
\usepackage{subfigure}
\usepackage{caption}
\usepackage{subcaption}
\usepackage{booktabs} %
\usepackage{float}
\usepackage{hyperref}

\usepackage{dsfont}
\usepackage{tikz}
\usetikzlibrary{arrows.meta, positioning, matrix, trees}

\usepackage{enumerate}

\newcommand{\R}{\mathbb{R}}

\newcommand{\eps}{\varepsilon}

\newcommand{\one}{\mathds{1}}

\usepackage[accepted]{icml2024}

\usepackage{amsmath}
\usepackage{amssymb}
\usepackage{mathtools}
\usepackage{amsthm}
\usepackage{comment}
\usepackage{algpseudocode}
\usepackage[capitalize,noabbrev]{cleveref}
\theoremstyle{plain}
\newtheorem{theorem}{Theorem}[section]
\newtheorem{proposition}[theorem]{Proposition}

\theoremstyle{definition}
\newtheorem{definition}[theorem]{Definition}
\newtheorem{assumption}[theorem]{Assumption}
\theoremstyle{remark}
\newtheorem{remark}[theorem]{Remark}

\usepackage[textsize=tiny]{todonotes}

\icmltitlerunning{Privacy Attacks in Decentralized Learning}

\begin{document}

\twocolumn[
\icmltitle{Privacy Attacks in Decentralized Learning}

\begin{icmlauthorlist}
\icmlauthor{Abdellah El Mrini}{yyy}
\icmlauthor{Edwige Cyffers}{xxx}
\icmlauthor{Aurélien Bellet}{zzz}
\end{icmlauthorlist}

\icmlaffiliation{yyy}{School of Computer and Communication Sciences, EPFL, Switzerland}
\icmlaffiliation{xxx}{Université de Lille, Inria, CNRS,
Centrale Lille, UMR 9189 - CRIStAL,
F-59000 Lille, France}
\icmlaffiliation{zzz}{Inria, Univ Montpellier, Montpellier, France}

\icmlcorrespondingauthor{Edwige Cyffers}{edwige.cyffers@inria.fr}

\icmlkeywords{Machine Learning, ICML}

\vskip 0.3in
]

\printAffiliationsAndNotice{}  %

\begin{abstract}
Decentralized Gradient Descent (D-GD) allows a set of users to perform collaborative learning without sharing their data by iteratively averaging local model updates with their neighbors in a network graph.
The absence of direct communication between non-neighbor nodes might lead to the belief that users cannot infer precise information about the data of others. In this work, we demonstrate the opposite, by proposing the first attack against D-GD that enables a user (or set of users) to reconstruct the private data of other users outside their immediate neighborhood.
Our approach is based on a reconstruction attack against the gossip averaging protocol, which we then extend to handle the additional challenges raised by D-GD. We validate the effectiveness of our attack on real graphs and datasets, showing that the number of users compromised by a single or a handful of attackers is often surprisingly large. We empirically investigate some of the factors that affect the performance of the attack, namely the graph topology, the number of attackers, and their position in the graph. Our code is available at \href{https://github.com/AbdellahElmrini/decAttack}{https://github.com/AbdellahElmrini/decAttack}.
\end{abstract}

\section{Introduction}

\looseness=-1 Federated learning \cite{kairouz2021advances} enables collaborative model
training among multiple data owners while keeping data
decentralized. Traditional federated algorithms employ a central server to
coordinate the training process and aggregate model updates \citep{FedAvg}.
However, this centralized approach poses challenges regarding its robustness
and scalability, as the server constitutes a single point
of failure and becomes a communication bottleneck as
the number of participants increases \citep{dec2017}. It also raises
privacy concerns: unless additional safeguards are employed, the central
server is able to observe the model updates of each participant, which has
been shown to sometimes leak as much information as the
data itself \citep{DeepLeakage,fedLearningAttacks,kariyappa2022cocktail}.

One way to mitigate these drawbacks is to consider fully
decentralized algorithms that eliminate the central server, relying instead on
peer-to-peer communications among participants, conceptualized as nodes within
a network graph \citep{boyd2006gossip}. Such algorithms, among which the
popular Decentralized Gradient Descent (D-GD), have received
increasing attention from the community 
\citep{scaman2017optimal,dec2017,hegeds_hal,
gossip_berthier,pmlr-v119-koloskova20a}. Beyond their enhanced scalability and
robustness, they prevent any single node from observing the individual
contributions of all other nodes. Indeed, while a node can observe the contributions
of its immediate neighbors in the graph, making the reconstruction of their
values somewhat expected \citep{pasquini2022privacy}, the contributions of
more distant nodes are iteratively mixed multiple times with other
contributions before being observed. Intuitively, one might thus expect that
individual contributions are harder to retrieve from only indirect feedback,
especially when the attacker node is far from its target. This, in
turn, leads to the common belief that decentralized learning may be more
privacy-preserving by design. Questioning this belief is the motivation of our work.

To do so, we design and evaluate attacks performed by a node (or a set of nodes) to retrieve other nodes' private data in two algorithms: gossip averaging, a
key protocol in decentralized computation \citep{boyd2006gossip,Xiao04,doi:10.1137/060678324}, and the prominent Decentralized
Gradient Descent (D-GD) algorithm \citep{dec2017,pmlr-v119-koloskova20a}.
As our goal is to assess the potential vulnerabilities
specific to decentralization, we focus on the strongest type of privacy
leakage, namely data reconstruction, executed by the weakest type of
attackers, namely honest-but-curious nodes. In other words, attacker nodes
follow the protocol and try to reconstruct as much data as possible from other
nodes using only their legitimate observations within the protocol. To the
best of our knowledge, the attacks we present are the first to
be able to reconstruct data from non-neighboring nodes.

Our attacks rely on interpreting each message received by the attackers as an
equation that ties together the private values of nodes, the weights of the
gossip matrix, and the received values. In the case of gossip averaging, we
show that an
appropriate factorization of a knowledge matrix representing these
equations yield the private values of the (often numerous)
reconstructible nodes, and additionally produces a reduced set of
equations linking together the values of the remaining nodes (which may leak
sensitive information in certain cases). Interestingly, we can predict
even before running the algorithm which nodes will have their values leaked
depending on the gossip matrix of the network graph.
For the case of D-GD, the key ideas of our attack against gossip averaging can
be applied
to reconstruct individual gradients, with some modifications to account for
the extra difficulty induced by combining gossip
averaging with gradient descent steps. Once we have reconstructed individual
gradients, we rely on existing gradient inversion attacks \citep{DeepLeakage,fedLearningAttacks} as a black-box
to
reconstruct data points. We show that under reasonable
assumptions, the reconstruction of the private data points of non-neighboring
nodes is still possible in D-GD.

We empirically evaluate the performance of our attacks on synthetic and
real graphs and datasets. Our attacks prove to be effective in practice
across various graph structures. In many cases, even a single attacker
node is able to reconstruct the data of a large number of other nodes, some of
them located many hops away. The collusion of several attacking
nodes further strengthens the attack. We also observe that the graph topology, the position of attackers in the graph, and the choice of learning rate play an important role.

Our results clearly show that relying solely on decentralization to ensure the
privacy of training data is ineffective, even if nodes
fully trust their immediate neighbors. Thus, additional protections must
be implemented.

\section{Related work}

\paragraph{Privacy in decentralized learning.} The potential for data compromise in decentralized learning is implicitly recognized by the research community, as evidenced by numerous works proposing defenses to limit privacy leakage. Among these are methods derived from differential privacy, where nodes introduce noise into their updates. Earlier work provides local differential privacy guarantees that come exclusively from the local noise addition~\citep{Huang2015a,admm_local,Bellet2018a,9524471}, while more recent results show that decentralization amplifies these baseline privacy guarantees under an adapted threat model~\citep{cyffers2022muffliato, cyffers2022privacy,Yakimenka2022StragglerResilientDD,admm}.
Other methods include %
the addition of correlated noise to constrain the information that can be extracted from local updates \cite{consensus2013, Mo2014, hanzely2017privacy, dellenbach2018hiding,10115431}. Our work shows that this area of research is indeed addressing a tangible threat. Regarding attacks on decentralized learning,  we are only aware of two recent works by \citet{pasquini2022privacy} and \citet{Dekker2023TopologyBasedRP}. The attack of \citet{pasquini2022privacy} only targets direct neighbors, making it quite similar to existing attacks on federated learning (see next paragraph). Concurrently and independently from our research, \citet{Dekker2023TopologyBasedRP} studied a different setting where nodes perform a sequence of secure aggregations with their neighbors, without considering a learning objective. Similar to \citep{pasquini2022privacy}, their attack focuses solely on direct neighbors.
In contrast, our attack is capable of reconstructing the data of more distant nodes, thus addressing the specific challenges of the decentralized setting more effectively.

\paragraph{Reconstruction attacks in federated learning.} Privacy attacks against (server-based) federated algorithms are an active topic of research. From the possibility of reconstructing data from the gradient via gradient descent \cite{DeepLeakage}, a.k.a. gradient inversion, various improvements have been developed \cite{fedLearningAttacks,zhao2020idlg, optAttack}, including techniques capable of separating gradients aggregated over large batches \cite{kariyappa2022cocktail}. In our attack on D-GD, our focus and core contribution is the reconstruction of gradients from the observations of attacker nodes. In a second step, we use a gradient inversion attack as a black box to reconstruct data from the reconstructed gradients.
Thus, our work is complementary to gradient inversion attacks and could benefit from any advancements in this area.

\section{Setting}

\subsection{Decentralized Algorithms}

In this section, we first present the setting of decentralized learning and the algorithms we study. We consider a fixed undirected and connected graph $\mathcal{G} = (\mathcal{V},\mathcal{E})$ where $|\mathcal{V}| = n$ and an edge $\{u,v\} \in \mathcal{E}$ indicates that $u$ and $v$ can exchange messages. We denote by $\mathcal{N}(u) = \{v : \{u,v\}\in \mathcal{E}\}$ the neighbors of node $v$. We assume that each node $u$ possesses a private value $x_u \in \R^d$. 
To ease the notations, we will assume in the presentation of the next sections that $d=1$, but the generalization to the vector case is straightforward. 

\begin{remark}
The simplification of considering local datasets with a single element was also made by \citet{pasquini2022privacy} and is natural in the case of decentralized averaging, where each node has indeed only one private value. For Decentralized Gradient Descent, recent work has proposed reconstruction attacks from gradients aggregated over several data points \cite{Boenisch2021WhenTC,kariyappa2022cocktail}, which can be readily used as a black-box in our attack. %
We thus abstract away this secondary aspect to focus on the particularity of decentralized learning.
\end{remark}

In this work, we consider synchronous gossip-based algorithms, where at each step each node performs a weighted average of its value with those of its neighbors, as determined by the weights given by a gossip matrix.

\begin{definition}[Gossip matrix] A gossip matrix $W \in [0,1]^{n \times n}$ over the graph $\mathcal{G}$ is a doubly stochastic matrix ($W \one = W^\top\one = \one$) with $W_{uv}>0$ if and only if there exists an edge between $u$ and $v$.
\end{definition}

\paragraph{Gossip averaging.} The gossip matrix can be used to iteratively compute the average of the private values $(x_v)_{v \in \mathcal{V}}$. Initializing each node $v$ with $\theta_v^0 = x_v$, the gossip averaging iteration \citep{Xiao04} is given by %
\begin{equation}
\label{eq:gossip_avg}
         \theta^{t+1} = W \theta^t.
\end{equation}
This converges to the global average at a geometric rate governed by the spectral gap of $W$ \citep{Xiao04}.

\begin{remark}[Accelerated gossip] An accelerated version of gossip \citep{gossip_berthier}, which involves replacing the multiplication by $W$ with a polynomial of $W$, is often used to speed up convergence. We note that this does not impact the information accessible to a node, as the values observed in accelerated gossip can be easily translated back to the non-accelerated counterpart through a simple linear transformation. For the sake of clarity, our discussion will thus focus on the non-accelerated version.
\end{remark}

\paragraph{Decentralized Gradient Descent (D-GD).} In decentralized learning, nodes aim to optimize an objective function of the form $f(\theta)=\sum_{v=1}^n L(\theta, x_v)$ where $\theta$ represents the parameters of the model and $L$ is some differentiable loss function.
The most popular algorithm to achieve this is D-GD \citep{dec2017,pmlr-v119-koloskova20a}. At each iteration of D-GD, each node performs a local gradient step with a learning rate $\eta>0$ followed by a gossip averaging step with its neighbors.
Let $\theta_v^0$ be an arbitrary initialization of the parameters at each node $v$. We denote the local gradient of node $v$ at iteration $t$ (scaled by $\eta)$ by
\begin{equation}
        g_v^t = - \eta \nabla_{\theta_v^t} L(\theta_v^t, x_v).
\end{equation}
Then, the gradient step of D-GD can be written as
\begin{equation}
    \theta^{t+\frac{1}{2}} = \theta^t + g^t,
    \label{eq:gd1}
\end{equation}
and the gossiping step corresponds to
\begin{equation}
    \theta^{t+1} = W \theta^{t+\frac{1}{2}}.
    \label{eq:gd2}
\end{equation}

Under suitable assumptions, D-GD converges to a global or local optimum of $f$ \citep{pmlr-v119-koloskova20a}.

\subsection{Threat Model}

Throughout this paper, we focus on honest-but-curious attackers that consist of a subset of nodes that adhere to the protocol but attempt to extract as much information as they can from their observations. We denote by $\mathcal{A}\subset \mathcal{V}$ the set of attacker nodes and by $\mathcal{N}(\mathcal{A})=\bigcup_{a\in \mathcal{A}}\mathcal{N}(a)\setminus \mathcal{A}$ the neighbors of the attackers. Without loss of generality, we assume that attacker nodes correspond to the first $|\mathcal{A}|$ nodes. When $|\mathcal{A}|>1$, we assume that the knowledge is completely shared between attacker nodes, even if the network graph does not contain edges between them.

We also assume that the network graph and the gossip matrix are known to the attacker. This assumption is often naturally satisfied in practical use cases (e.g.,  within social networks), and we further discuss its relevance in \cref{app:pubW}.

\section{Reconstruction in Gossip Averaging}
\label{sec:avg}

\begin{figure*}[t]
    \centering
    \includegraphics[width=\textwidth]{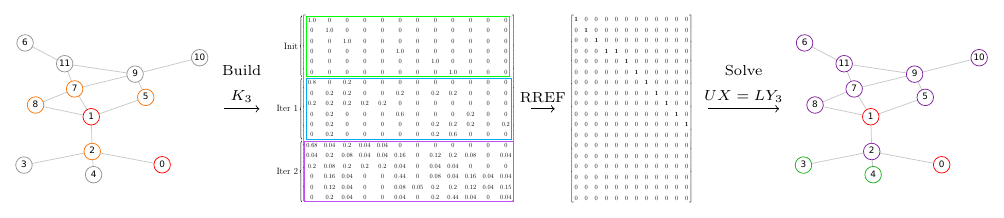}
    \caption{Overview of our attack on gossip averaging. The attackers $0$ and $1$ (red) receive updates from nodes $2$, $5$, $7$ and $8$ (orange). For $T=3$ iterations, it leads to the knowledge matrix $K_3$. Its RREF (matrix $U$) exhibits that only nodes $3$ and $4$ are non-reconstructible (green). All other nodes (purple) have their private value leaked.}
    \label{fig:gossipreconstruction}
\end{figure*}

In this section, we describe our attack on gossip averaging. The key idea of our attack is that each message $\theta^t_v$ received by an attacker node $a\in \mathcal{A}$ from one of its neighbors $v\in \mathcal{N}(a)$ corresponds to a linear equation where the unknown are the private values of the nodes of the graph and the coefficients depend on the gossip matrix $W$ (which is known to the attackers).
Our attack consists in accurately generating this system of linear equations ($K_T X = Y_T$ in our notations) and then solving it in order to reconstruct as many private values as possible. We provide a visual summary of this reconstruction process in \cref{fig:gossipreconstruction}.

\paragraph{Collecting linear equations.} For a given gossip matrix $W$ and set of attackers $\mathcal{A}$, we now describe how the attackers can systematically construct a system of linear equations capturing the knowledge that they gather about the private values throughout $T$ iterations of gossip averaging. Informally, the attackers receive $|\mathcal{A}| + T \cdot |\mathcal{N}(\mathcal{A})|$ values, and they can associate each of these values to a linear combination of the $n$ private inputs. At the beginning of the protocol, the attackers already know their inputs (these are the first $|\mathcal{A}|$ values). Then, at each round $t$ of gossip, the attackers receive one value $\theta_v^{(t)}$ from each of their neighbors $v\in \mathcal{N}(\mathcal{A})$. These received values correspond to a linear combination of private inputs, where the weights are given by the corresponding powers of the gossip matrix.

Formally, we denote this system of linear equations by $K_T X = Y_T$, where $X=(x_0,\dots,x_n)\in\R^n$ is the vector of private values that the attackers seek to reconstruct, and $K_T,Y_T$ are defined as follows.
\begin{definition}[Knowledge matrix and observation vector]
The \emph{knowledge matrix} $K_T\in\R^{{|\mathcal{A}| + T \cdot |\mathcal{N}(\mathcal{A})|}\times n}$ is defined by \cref{alg:knMatrix}. The \emph{observation vector} $Y_T \in \R^{|\mathcal{A}| + T \cdot |\mathcal{N}(\mathcal{A})|}$ is obtained by having attackers stack their own values and the received messages $(\theta^t_v)_{0\leq t\leq T-1,v\in \mathcal{N}(\mathcal{A})}$. The pair $(K_T, Y_T)$ forms the view of the attackers.
\end{definition}
Note that $K_T$ only depends on the gossip matrix, whereas $Y_T$ is specific to the private values.

For concreteness, let us consider a simple example. Consider a graph of $n$ nodes with one attacker at position $0$. The attacker knows his own private input $x_0$ (stored in the first entry of $Y_T$), which corresponds to the insertion of the one-hot vector $(1,0,\dots, 0)$ in the first row of the knowledge matrix $K_T$. Assuming that nodes $1$ and $2$ are the neighbors of the attacker, they will send $\theta_1^{(0)}=x_1$ and $\theta_2^{(0)}=x_2$ to the attacker at iteration $0$ (stored in the second and third entries of $Y_T$), corresponding to the insertion of $(0,1,0,\dots, 0)$ and $(0,0,1,0,\dots, 0)$ respectively in the second and third rows of $K_T$. At iteration $1$, they will send $\theta_1^{(1)}=\sum_{j} W_{j,1} x_j$ and $\theta_2^{(1)}=\sum_{j} W_{j,2} x_j$, corresponding to $(W_{0,1}, \dots, W_{n-1,1})$ and $(W_{0,2}, \dots, W_{n-1,2})$ respectively. In general, for each iteration $t$, the attacker receives the values $\theta_1^{(t)}=\sum_{j} W^t_{j,1} x_j$ and $ \theta_2^{(t)}=\sum_{j} (W^t)_{j,2} x_j $ from his two neighbors, by definition of the gossip averaging algorithm, and this information is stored in the corresponding rows of $K_T$ and $Y_T$.

\paragraph{Solving the linear system.} Recovering the private values corresponds to solving the equation $K_T X=Y_T$ where $X$ is the unknown. The set of solutions of this equation is always non-empty by construction (as the set of private values satisfies the equation), and is reduced to a single element when $K_T$ is full rank (i.e., rank $n$). Thus, if $K_T$ is full rank, attackers can reconstruct all the private values of the nodes. Otherwise, attackers may still reconstruct a subset of the private values, and deduce relationships between the ones that cannot be fully reconstructed.

To do so, we factorize $K_T$ using its Reduced Row Echelon Form (RREF), namely $K_T = L^{-1}U$ where $U$ is the unique RREF of $K_T$ and $L$ is such that $U X = L Y_T$.
RREF is a form often introduced in algebra courses to teach Gauss-Jordan elimination \citep{rref}. In this form, the block of matrix $U$ corresponding to reconstructible nodes becomes the identity matrix, while the remaining rows (corresponding to non-reconstructible nodes) contain the linear equations that link their values together. Solving $K_T X = Y_T$ is thus equivalent to solving $U X = L Y_T$ with trivial equations $1 \times X_v = (LY_T)_v$ for reconstructible nodes.
Hence, this decomposition allows us to clearly identify the nodes that our attack can reconstruct, even before the algorithm is executed (as the attackers only need to construct $K_T$ but not $Y_T$).

\begin{algorithm}
\caption{Knowledge matrix construction}\label{alg:knMatrix}
\begin{algorithmic}[1]
\State \textbf{Inputs:} the graph $\mathcal{G}$, the set of attackers $\mathcal{A}$, the number of iterations $T$
\State \textbf{Initialization:} $K_T$ an empty matrix of size $m\times n$ where $m = |\mathcal{A}| + T \cdot |\mathcal{N}(\mathcal{A})|$
\For {each $v \in \mathcal{A}$}
    \State $K_T[v,:] \leftarrow e_v$, where $e_v$ is the one-hot vector of size $n$ with $1$ at index $v$
\EndFor
\State $i \gets |\mathcal{A}|$
\For{ $t$ from $0$ to $T-1$}
    \For{each $v \in \mathcal{N}(\mathcal{A})$}
        \State $K_T[i,:] \gets W^t[v,:]$ %
        \State $i \gets i + 1$
    \EndFor
\EndFor
\State \textbf{Return:} $K_T$
\end{algorithmic}
\end{algorithm}

\begin{definition}[Reconstructible node] \label{d:indistinct}
Given a network graph $G$ and a set of attackers $A$, a node $v$ is said to be reconstructible by $\mathcal{A}$ after $T$ iterations if in the RREF form $U$ of $K_T$, the row corresponding to $v$ is a vector with the value $1$ in a single entry and $0$ everywhere else.
\end{definition}

    A complete characterization of reconstructible nodes using explicit graph-related quantities, instead of linear algebra as in Definition~\ref{d:indistinct}, appears to be quite challenging.
    In \cref{app:counter}, we provide some counter-intuitive examples on simple graphs to illustrate this point.

\begin{remark}[Secure aggregation]
    \label{rmk:secagg}
    One of the defense mechanisms widely studied in federated learning is the use of secure aggregation (SecAgg), where a set of parties jointly compute the sum of their private values without revealing more than the final output \citep{Bonawitz2017a}. This could be used in gossip averaging: at each iteration, each node could compute the weighted sum by running a SecAgg protocol with its neighbors \citep{Dekker2023TopologyBasedRP}. Aside from the high computation and communication overhead, in terms of leakage, this would be akin to an attack in a model without SecAgg from a node only connected to the true attacker. We note that our attack still works in this case, with a slight modification in the construction of the knowledge matrix. As illustrated by \cref{fig:line_graph_MNIST} or \cref{fig:ego414}, numerous nodes can have their data leaked in the case where the attacker has a single neighbor.
\end{remark}

\begin{remark}[Extension to dynamic networks]
\looseness=-1 Our approach can be adapted to dynamic networks where the graph changes over time, as long as the attackers know the gossip matrix $W_t$ used at each iteration. Indeed, the reconstruction problem is equivalent to the one with a static $W$, up to the modification of the construction of the knowledge matrix. In some cases, dynamic networks might be more vulnerable to reconstruction: the union of the direct neighbors of the attackers could be larger than for static networks, so the proportion of nodes that can be directly reconstructed might increase.
\end{remark}

\section{Reconstruction in D-GD}

In this section, we present our attack on Decentralized Gradient Descent. Our attack proceeds in two steps: we first reconstruct the gradients and then reconstruct the data points from the reconstructed gradients. The latter step can be done by resorting to existing gradient inversion attacks as a black box \citep[see e.g.,][]{DeepLeakage,fedLearningAttacks,kariyappa2022cocktail}. Therefore, in the rest of the section, we focus on the gradient reconstruction part, which is the core of our contribution.

To reconstruct the gradients of nodes, our attack builds upon the attack on gossip averaging presented in Section~\ref{sec:avg}. However, several additional challenges arise. While the private values in gossip averaging remain constant across iterations, the gradients in D-GD evolve over time. This means that each step brings new unknowns in the equation, making it impossible to find an exact solution through direct equation solving.  Attacking D-GD thus requires additional steps:
\begin{enumerate}[(i)]
    \item Reducing the number of unknowns by introducing similarity assumptions on the gradients;
    \item Changing the construction of the knowledge matrix, to reconstruct the gradients $g_v^t$ instead of the model parameters $\theta^t$;
    \item Removing the attackers' own contributions to reduce overall noise in the approximated reconstruction;
\end{enumerate}

\paragraph{(i) Gradient similarity.} We derive our reconstruction attack under the assumption that the gradients of a node along the iterations can be described as a combination of a fixed and a random component.

\begin{assumption}[Noise-signal gradient decomposition]
    For each node $v \in \mathcal{V}$, we assume that we can decompose the gradient update as follows
    \begin{equation}
        g_v^t = - \eta \nabla_{\theta_v^t} L(\theta_v^t, x_v) = g_v + N_v^t, 
    \end{equation}
    where $N_v^t$ is a centered random variable with variance $\sigma^2$, and the constant part $g_v$ is specific to node $v$ but stays the same across iterations.
    \label{assum:gau}
\end{assumption}

We note that this assumption is typically not satisfied in real use cases, but we will see in Section~\ref{sec:exp} that the algorithm is robust in practice to small violations of this assumption (in particular, it works well when gradients change sufficiently slowly across iterations).

\begin{algorithm}[t]
\caption{Building the knowledge matrix for D-GD}\label{alg:knMatrixgd}
\begin{algorithmic}[1]
\State \textbf{Inputs:} the graph $\mathcal{G}$, the set of attackers $\mathcal{A}$, the set of targets $\mathcal{T} = \mathcal{V}\setminus \mathcal{A}$, the number of iterations $T$
\State \textbf{Initialization:} $K_T$ an empty matrix of size $m\times n$ where $m = T \cdot |\mathcal{N}(\mathcal{A})|$.

\State $i \gets 0$
\For{ $t$ from $0$ to $T-1$}
    \For{each $v \in \mathcal{N}(\mathcal{A})$}
        \State $K_T[i,:] \gets (\sum_{j=0}^{t} W_{\mathcal{T},\mathcal{T}}^j)[v-|\mathcal{A}|,:] $ %
        \State $i \gets i + 1$
    \EndFor
\EndFor
\State \textbf{Return:} $K_T$
\end{algorithmic}
\end{algorithm}

\begin{remark}[Connection to differential privacy]
    Assumption~\ref{assum:gau} naturally models situations where noise is added to satisfy differential privacy \citep{Huang2015a,9524471,cyffers2022muffliato}.
    In particular, one can see this as an instance of averaging under local differential privacy. Naturally, the accuracy of the reconstruction will be directly related to the variance of the noise, and thus to the chosen privacy budget.
\end{remark}

\paragraph{(ii) Knowledge matrix construction.}
Denoting the set of target nodes as $\mathcal{T}=\mathcal{V}\setminus \mathcal{A}$, we rewrite the gossip matrix $W$ as follows:
\begin{equation}
    W = \begin{pmatrix}
W_{\mathcal{A},\mathcal{A}} & W_{\mathcal{A},\mathcal{T}} \\
W_{\mathcal{T},\mathcal{A}} & W_{\mathcal{T},\mathcal{T}}
\end{pmatrix}
\end{equation}

We now write the values $\theta^{t+\frac{1}{2}}$ shared by nodes during the execution of the algorithm in terms of this decomposition.

\begin{proposition}[Closed-form of D-GD updates]
    For the D-GD algorithm described by \eqref{eq:gd1} and \eqref{eq:gd2}, we have:
\begin{equation}
    \theta^{t+\frac{1}{2}} = \begin{pmatrix}
    \theta^{t+\frac{1}{2}}_\mathcal{A} \\ 
    \left(\sum\limits_{i=0}^{t} W_{\mathcal{T},\mathcal{T}}^i\right) g_\mathcal{T} + 
    \sum\limits_{i=0}^{t} W_{\mathcal{T},\mathcal{T}}^i N_T^{t-i} \\
    + \sum\limits_{i=0}^{t-1} W_{\mathcal{T},\mathcal{T}}^{t-1-i} W_{\mathcal{T},\mathcal{A}} \theta^{i+\frac{1}{2}}_\mathcal{A}
    \end{pmatrix}.
\end{equation}
\end{proposition}
\begin{proof}
    The proof is done by induction, applying the \cref{eq:gd1} and \cref{eq:gd2} and rearranging the terms.
\end{proof}
\looseness=-1 This formula leads to a more complex computation of the knowledge matrix $K_T$ (see \cref{alg:knMatrixgd}) compared to the one used for gossip averaging, because we want to reconstruct the gradients $g=(g_v)_{v\in \mathcal{V}}$ (not model parameters $\theta^t$).

\paragraph{(iii) Attackers' contributions removal.}
Given $Y_T$ the concatenated vector of updates received by the attackers until iteration $T$, the attackers need to preprocess it in order to remove their own contributions. \Cref{alg:prep} shows how to compute $\hat{Y}_T$ from $Y_T$ and the gossip matrix.  

\begin{algorithm}[t]
\begin{algorithmic}
	
\caption{Removing the attackers' contributions} \label{alg:prep}

\State \textbf{Inputs :} the gossip matrix $W$ of the graph $\mathcal{G}$, the set of attackers $\mathcal{A}$, the set of targets $\mathcal{T} = \mathcal{V} \setminus \mathcal{A}$, the number of iterations $T$, the dimension of the model $d$, the received updates $Y_T$ and the concatenated vector of the updates sent by the attackers $\theta_\mathcal{A} = (\theta_\mathcal{A}^{\frac{1}{2}}, \dots, \theta_\mathcal{A}^{T-\frac{1}{2}})$.

\State Initialize $\hat{Y}_T \in \mathbb{R}^{T  \times |\mathcal{N}(\mathcal{A})| \times d}$ 
\State Initialize $B \in \mathbb{R}^{|\mathcal{T}|\times d} $ with zeros

\For{ $t \in 0,1, \dots , T-1$}
    \State $\hat{Y}_T[t,:] \leftarrow Y_T[t,:] - B[\mathcal{N}(\mathcal{A}),:] $ 
    \State $B \gets W_{\mathcal{T},\mathcal{T}} B + W_{\mathcal{T},\mathcal{A}} \theta_\mathcal{A}^{t+\frac{1}{2}} $ \Comment{The contribution of the attackers to be eliminated}
\EndFor

\State \textbf{Return :} $\hat{Y}_T$
\end{algorithmic}
\end{algorithm}

\begin{figure*}[ht]
    \centering
    \includegraphics[width=.99\textwidth]{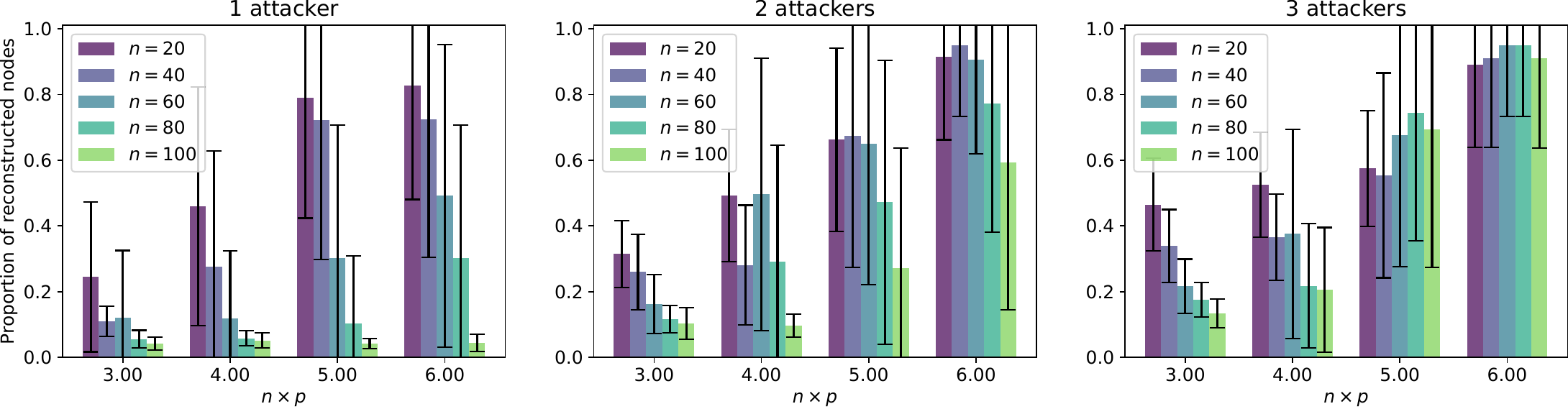}
        \caption{Average fraction of reconstructed nodes in Erdös-Rényi graphs with a different number of nodes $n$ and edge probability $p$, for $1,2$ or $3$ attacker nodes. Error bars give the standard deviations, computed over $20$ random graphs.}
    \label{fig:ER_recons}
\end{figure*}

\paragraph{Gradient reconstruction.}
Equipped with the previous concepts, reconstructing the gradients reduces to a Generalized Least Square (GLS) problem:
\begin{equation}
    K_T g + \eps_T= \hat{Y}_T,
\end{equation}
where $\eps_T$ is the noise of covariance $\Sigma_T$, the (non-diagonal) covariance matrix resulting from the aggregation of noise from various nodes at each step. The formula and detailed algorithm for computing this covariance matrix is provided in \cref{app:cov} (\cref{app:covMatrix}).

The reconstruction minimizing the squared error is thus given by:
\begin{equation*}
    \hat{g} = (K_T^\top \Sigma_T^{-1} K_T)^{-1} K_T^\top \Sigma_T^{-1} \hat{Y}_T.
\end{equation*}
and this estimator is unbiased with a variance of size $(K_T^\top \Sigma_T^{-1} K_T)^{-1}$ under \cref{assum:gau}.

\begin{remark}[Impact of covariance matrix]
    An alternative to computing this exact covariance matrix $\Sigma_T$ is to solve directly the Ordinary Least Square (OLS). Although this gives a bit more weight to the most noisy points compared to the optimal estimator under \cref{assum:gau}, we found experimentally that the reconstruction quality is quite similar between the two methods. This can be explained by the fact that the assumption of the noise structure is not fully satisfied, and OLS tends to be quite robust in practice.
\end{remark}

\begin{remark}[Gossip protocols with consecutive averaging steps]
\looseness=-1 An existing variant of D-GD involves performing multiple gossip averaging steps for each gradient step \cite{pmlr-v139-kong21a,pmlr-v119-koloskova20a,cyffers2022muffliato, dandi2022dataheterogeneityaware}. Our attack can be easily adapted to this case. In fact, it makes reconstruction easier as it brings D-GD closer to gossip averaging by effectively making gradients constant across several iterations. In scenarios where sufficiently many iterations are performed, we can simply use the reconstruction attack designed for gossip averaging, up to the minor modifications to the knowledge matrix.
\end{remark}

\section{Experimental Results}
\label{sec:exp}

In this section, we show that our attacks on gossip averaging and decentralized gradient descent are effective in practice. We evaluate our attacks on synthetic and real-world graphs, showing successful reconstructions in all cases. Our code is available at \href{https://github.com/AbdellahElmrini/decAttack}{https://github.com/AbdellahElmrini/decAttack}.

\subsection{Gossip Averaging}

\begin{figure}[t]
\centering
\includegraphics[width=0.78\linewidth]{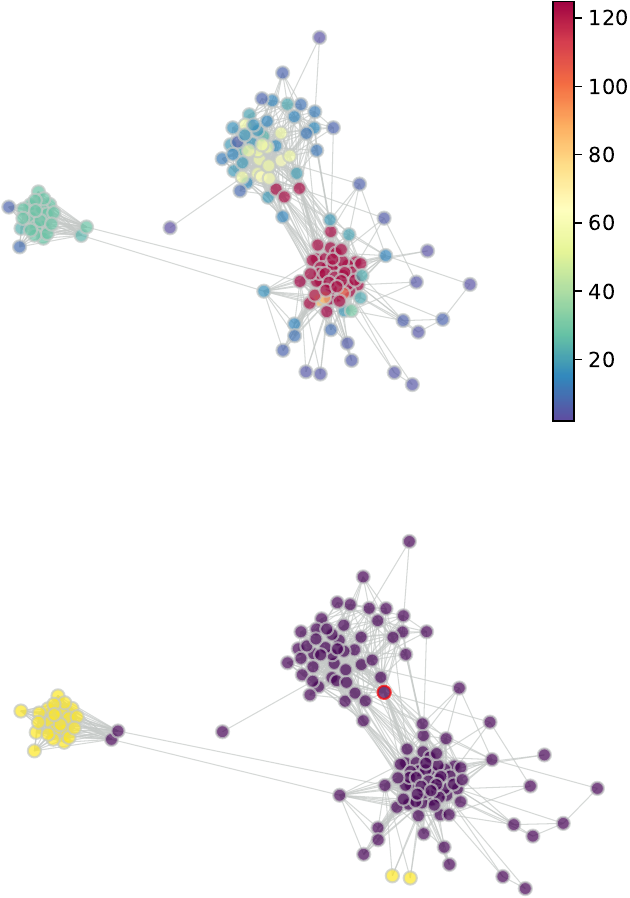}
\caption{Reconstruction attack on the Facebook Ego Graph $414$. Top: each node is colored by the number of nodes it can reconstruct among the $147$ other nodes. Bottom: detailed view of the case where the node circled in red is the attacker, with reconstructed nodes shown in purple and non-reconstructed ones in yellow.}
\label{fig:ego414}
\end{figure}

\paragraph{Synthetic graphs and impact of global characteristics.} We generate Erdös-Rényi graphs with different number of nodes $n$ and different edge probability $p$. We also vary the number of attacker nodes from $1$ to $3$. The proportion of reconstructed nodes in each setting is shown in Figure~\ref{fig:ER_recons}. We can see that a single attacker node is typically able to reconstruct many nodes, well beyond its direct neighborhood. The fraction of reconstructed nodes increases with the connectivity of the graph and the number of attackers. 

\paragraph{Real-world graphs.} We consider the graphs of the Facebook Ego dataset \cite{egofbgraph}, where nodes are the friends of a given user (this central user is not present in the graph) and edges encode the friendship relation between these nodes. These graphs typically present several communities corresponding to different interests. We show that reconstruction is much more likely within clusters, but also occur across nodes belonging to distinct clusters. We give an example in \cref{fig:ego414} and report other Ego graphs in \cref{app:ego}.

\begin{table}[th!]
\centering
\caption{Correlation between the centrality of the attacker and the proportion of nodes it is able to reconstruct.}
\begin{tabular}{@{}l@{\hspace{10pt}}c@{\hspace{20pt}}r@{\hspace{10pt}}c@{}}\toprule
Centrality & Erdos-Renyi graph & Ego graph \\
\midrule
Degree & 0.94 & 0.94\\
Eigenvector & 0.81 & 0.64 \\
Betweenness & 0.84  & 0.66 \\
\bottomrule
\end{tabular}
\label{tab:centrality}
\end{table}

\paragraph{Impact of nodes' characteristics.} Intuitively, it is easier to attack close nodes rather than distant ones. We quantify this effect by evaluating how the attacker centrality impacts its reconstruction ability. Centrality measures are often used in graph mining to measure how important a node is. We test two kinds of graphs. First, we randomly sample Erdös-Renyi graphs with $n = 50$ and $p=0.08$. We reject graphs that are not fully connected and consider a single attacker (node $0$ as the graph construction treats all nodes equally). Second, for a fixed Facebook Ego graph, we make each node play the role of the attacker in turn.
To assess the relation between the centrality of a node and the proportion of the nodes it is able to reconstruct, we use Spearman correlation, a non-parametric measure which is only based on rank statistics. We observe in \cref{tab:centrality} that for both types of graphs, degree centrality is the most correlated with the proportion of reconstructed nodes. Interestingly, other centrality measures that capture some structural properties of the graph beyond immediate neighbors, such as eigenvector and betweenness centrality, also exhibit strong correlation.
In \cref{app:locrec}, we report metrics to assess how the relative position of the attacker and its target impact the probability of reconstruction.

\subsection{Decentralized Gradient Descent}

\begin{figure*}[tbh!]
    \centering
    \begin{subfigure}[Cifar10, logistic regression, learning rate $10^{-4}$]{
        \centering
        \begin{tikzpicture}
            \node[inner sep=0pt] (image1) at (0,0)
                {\includegraphics[width=1\textwidth]{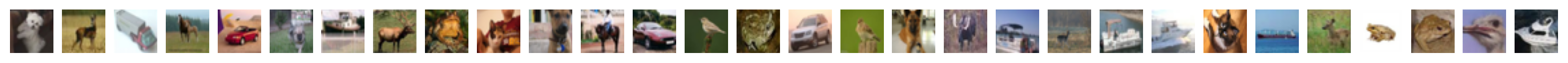}};
            \node[inner sep=0pt] (image2) at (0,-.8)
                {\includegraphics[width=1\textwidth]{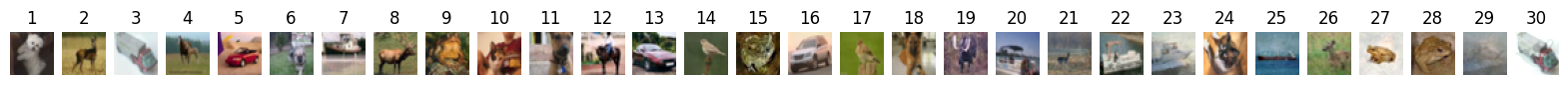}};
        \end{tikzpicture}
        \label{fig:line_graph_Cifar10}
        }
    \end{subfigure}%
    \begin{subfigure}[MNIST, convnet, learning rate $10^{-6}$, gradient inversion from \cite{fedLearningAttacks}]{
        \centering
        \begin{tikzpicture}
            \node[inner sep=0pt] (image3) at (0,0)
                {\includegraphics[width=\textwidth]{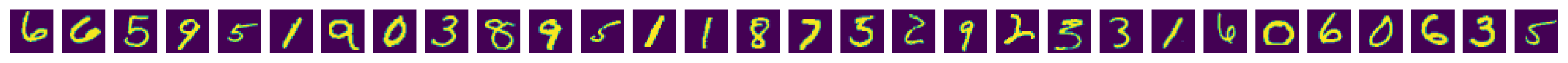}};
            \node[inner sep=0pt] (image4) at (0,-.8)
                {\includegraphics[width=\textwidth]{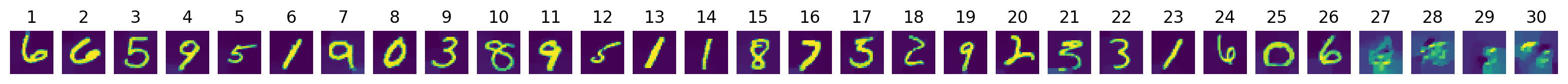}};
        \end{tikzpicture}
        \label{fig:line_graph_MNIST}
        }
    \end{subfigure}
    
    \caption{Reconstruction attack on D-GD for a line graph with $31$ nodes where the attacker lies at an extremity. The first (resp. second) row shows the true (resp. reconstructed) inputs of the 30 other nodes ordered by their distance to the attacker.}
    \label{fig:line_graphs}
\end{figure*}

We now turn to the more challenging case of D-GD. We first focus on the Cifar10 dataset \cite{Krizhevsky2009LearningML} using a model that consists of a fully connected layer with a softmax activation, a bias term, and a cross-entropy loss (a.k.a., logistic regression). For this simple model, one can reconstruct a data point from its gradient in closed-form \citep[see][Lemma 6.1 therein]{biswas}. This allows us to focus the evaluation on the core of our attack (reconstructing gradients), avoiding the inherent errors due to the imperfections of gradient inversion attacks on more complex models. We start running our attack when the model is close to convergence so that gradients are more stable. To ensure that attackers gather enough information about other nodes in the knowledge matrix, we run D-GD for a number of steps roughly equal to the diameter of the graph.

\begin{figure*}[tbh!]
    \centering
    \includegraphics[height=.35\textwidth]{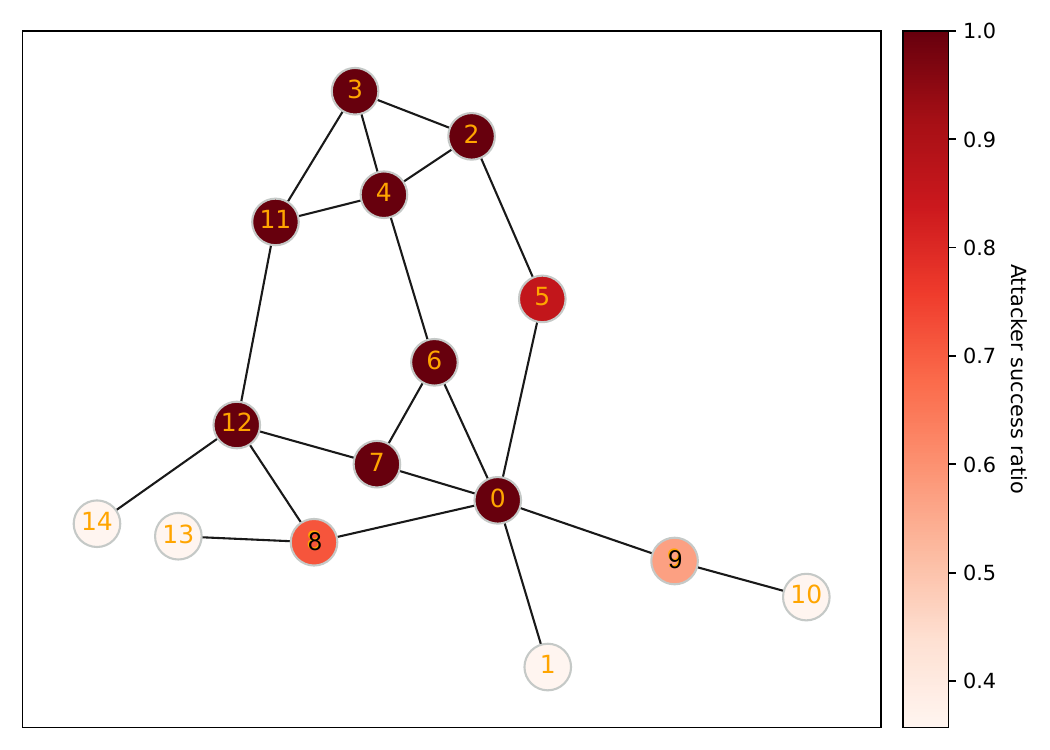}
    \hspace{.05\textwidth}
    \includegraphics[height=.35\textwidth]{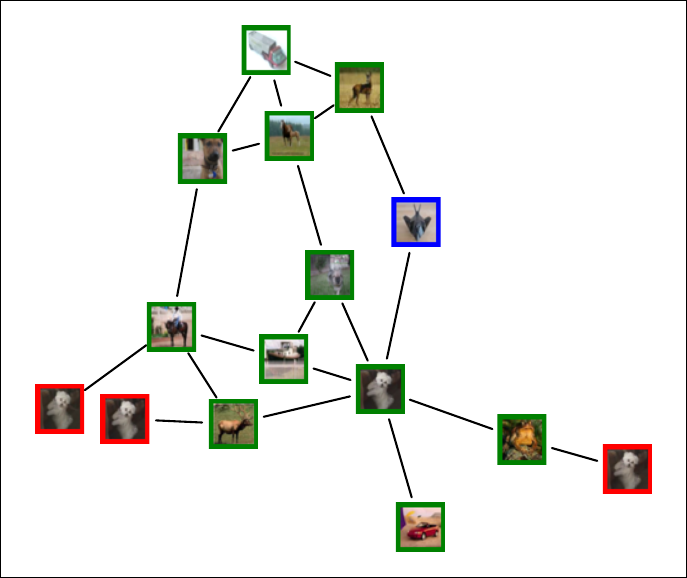}
    \caption{Reconstruction attacks on D-GD for the Florentine graph (Cifar10, logistic regression model, learning rate $10^{-5}$). Left: the color of each node represents the success rate when that node is the attacker. The success rate is measured as the fraction of nodes for which the reconstructed image achieves a PSNR superior to $10$ with respect to the original image (averaged over 10 experiments).
    Right: detailed view of the case where the attacker is node 5 (highlighted with blue borders). Nodes with green borders are accurately reconstructed, the ones with red borders are not. For completeness, the true input images are shown in Appendix~\ref{app:expe_dgd}.}
    \label{fig:ReconstructionFlorentine}
\end{figure*}

We first run our attack on the classic Florentine graph \citep{florentine}, a graph with $n=15$ nodes describing marital relations between families in 15th-century Florence. We can see in \cref{fig:ReconstructionFlorentine} that most nodes (except those located at the edge of the network) can reconstruct a large proportion of other nodes with very good visual accuracy.

\begin{figure}[tbh!]
    \centering
    \includegraphics[width=0.45\textwidth]{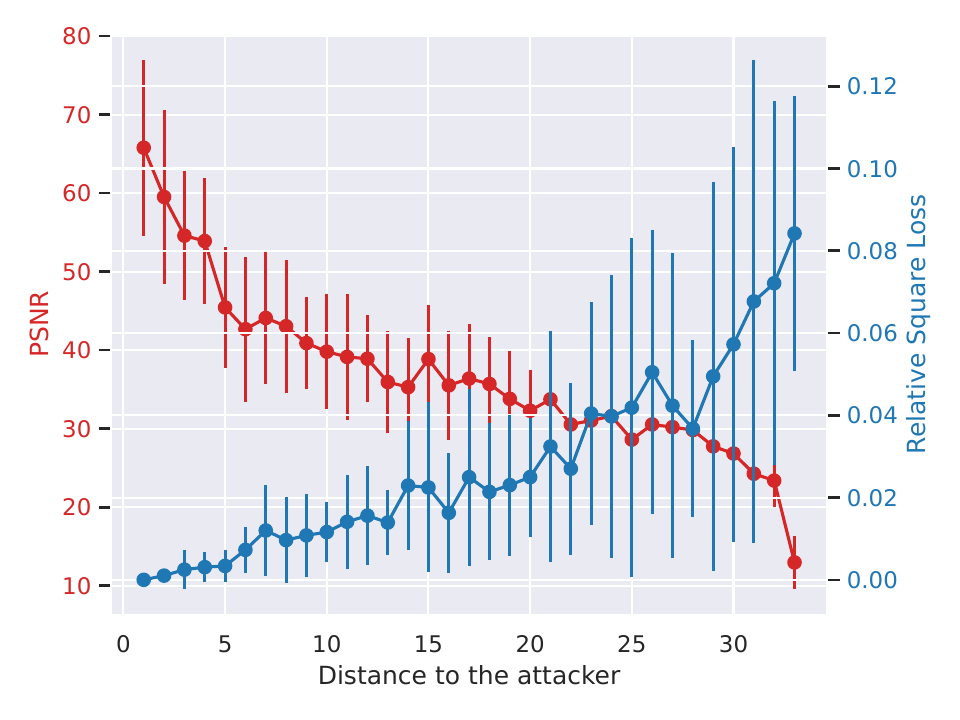}
    \caption{Reconstruction accuracy (measured by PSNR) and error (measured by the relative square distance) as a function of the distance between the victim and the attacker for D-GD on a line graph (see details in Figure~\ref{fig:line_graph_Cifar10}). The plot shows the mean and standard deviation across $100$ experiments.}
    \label{fig:psnr_line}
\end{figure}

We further test the limit of reconstruction by using a line graph of $31$ nodes with the attacker at an extremity. We see in Figure~\ref{fig:line_graph_Cifar10} that the results are far better than one would intuitively expect: even though the gradients of distant nodes are mixed many times before reaching the attacker, our attack allows to disentangle the contributions of the different nodes to enable informative reconstruction up to distance 28. The visual impression is further confirmed by reconstruction metrics over multiple runs (see \cref{fig:psnr_line}).

Finally, we switch to a more complex model for which it is necessary to rely on a gradient inversion attack. We use a small convolutional neural network (see details in Appendix~\ref{app:details}) on the MNIST dataset and the gradient inversion attack of \citet{fedLearningAttacks} as a black box. Using a line graph with $31$ nodes as in the previous experiment, we see in Figure~\ref{fig:line_graph_MNIST} that our approach can naturally rely on a black-box gradient inversion attack to reconstruct data from the gradients of more complex models. Here, the reconstructions are accurate up to distance 26. We refer to Appendix~\ref{app:expe_dgd} for results on the Florentine graph.

We note that the performance of our attack on D-GD is sensitive to several parameters. First, having similar local parameters $\theta_v$ across nodes enables better reconstructions as the observed values are primarily influenced by the gradients rather than by variations in the parameters. This condition is easily satisfied either by initializing all the nodes with the same parameters (a standard practice in D-GD) or by waiting until the system approaches convergence. Second, the learning rate plays an important role: it should be small enough to ensure that gradients do not vary wildly across iterations.
We illustrate this behavior in Appendix~\ref{app:learning_rate}.

\section{Conclusion}

Our work demonstrates the vulnerability of data when using standard decentralized learning algorithms. More precisely, we show that a node can attack and successfully reconstruct the data of non-neighboring (and sometimes quite distant) nodes by leveraging the communication structure inherent to gossip protocols. We also highlight the impact of the graph topology and the position of the attackers in the success rate of the attack in practice. An interesting open question is a full characterization of reconstructible nodes by structural properties of the graph, which appears to be a challenging problem as we discuss in Appendix~\ref{app:counter}. Likewise, optimizing the graph so as to minimize the number of reconstructible nodes is an open question.

Our work shows that one cannot rely on decentralization alone to protect sensitive data. Therefore, to provide robust privacy guarantees, decentralized algorithms must be combined with additional defense mechanisms such as those based on differential privacy \citep{cyffers2022muffliato}.
In future work, we would like to formally and empirically analyze the relation between differential privacy guarantees and the success rate of our reconstruction attacks.

\section*
{Impact Statement}
Our work designs privacy attacks in decentralized learning. Given the ethical challenges in justifying such attacks in real-world scenarios, and considering that the deployment of our attack could potentially cause harm, we primarily focus on raising awareness about privacy risks in decentralized learning by using public datasets. Our findings challenge the common misconception that decentralization inherently brings privacy, particularly for distant nodes. Interestingly, as mentioned in the paper, the construction of the knowledge matrix $K_T$ is independent from the actual private data, and could thus be used as part of an auditing phase to measure how much information would be leaked if one were to use a specific gossip matrix/graph. We believe that sharing these insights is crucial and contributes to the development of safer approaches in machine learning. As highlighted in the introduction and motivation our work should be seen as an empirical illustration of the need for more robust privacy protections, such as the ones offered by differential privacy.

\section*{Acknowledgments}

This work was supported by grant ANR-20-CE23-0015 (Project PRIDE), the
ANR-20-THIA-0014 program ``AI\_PhD@Lille'' and the ANR 22-PECY-0002 IPOP 
(Interdisciplinary Project on Privacy) project of the Cybersecurity PEPR.

Experiments presented in this paper were carried out using the Grid'5000 testbed, supported by a scientific interest group hosted by Inria and including CNRS, RENATER and several Universities as well as other organizations (see https://www.grid5000.fr).

\bibliography{biblio}
\bibliographystyle{icml2024}

\newpage
\appendix
\onecolumn

\begin{center}
{\Large\bf Supplementary Material}
\end{center}

\section{Examples of Reconstructible Nodes}
\label{app:counter}

We illustrate the difficulty of knowing which nodes are reconstructible from explicit graph characteristics on some small examples reported in \cref{fig:counter}. The first graph $G1$ is the smallest one, but has the smallest proportion of reconstructible nodes ($2/4$). Adding a single node as done in $G3$ makes all the nodes become reconstructible. However, adding a new node can also drastically reduce the proportion of reconstructible nodes, as shown in $G4$. Symmetry is not enough to prevent reconstruction, as highlighted in $G2$ where all nodes stay reconstructible.

\begin{figure}[h]
    \centering
    \includegraphics[width=.9\textwidth]{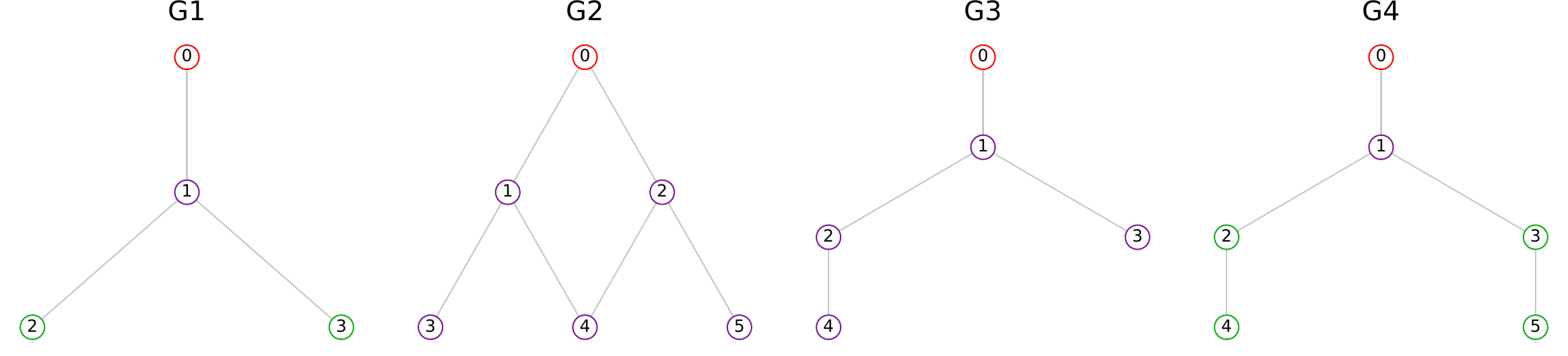}
    \caption{Examples of similar graphs with different reconstructible sets. Attacker is node $0$ (red), reconstructible nodes are in purple, and non-reconstructible ones are in green.}
    \label{fig:counter}
\end{figure}
\section{Construction of the Covariance Matrix for the Reconstruction Attack on Decentralized Gradient Descent}
\label{app:cov}

\Cref{app:covMatrix} shows how to build the covariance matrix $\Sigma^T$, necessary for the GLS method. 
\begin{algorithm}
\begin{algorithmic}
\caption{Building the covariance matrix} \label{app:covMatrix}

\State \textbf{Inputs :} the graph $\mathcal{G}$, the set of attackers $\mathcal{A}$, the number of iterations $T$, $\sigma$ the noise amplitude. 
\State \textbf{Initialization :} $\Sigma^T$ an empty matrix and of size $m\times m$ where $m = T \cdot |\mathcal{N}(\mathcal{A})|$.

\State $ i = 0$

\For{ $t \in 0, \dots ,T-1$}
    \For{ $v \in \mathcal{N}(\mathcal{A})$} 
        \State $j =  t \cdot |\mathcal{N}(\mathcal{A})|$
        \For{ $t' \in t, \dots , T-1$}
            \For{ $v' \in \mathcal{N}(\mathcal{A})$} 
                \State $C[i,j] = \sigma^2(\sum_{l=0}^T W_{\mathcal{T},\mathcal{T}}^{t+t'-2l})[v,v']$  
                \State $C[j,i] = C[i,j] $
                \Comment{Entry $(i,j)$ here corresponds to the pair $((t,v),(t',v'))$} 
                \State $j = j+1$
            \EndFor
        \EndFor
        
        \State $i = i + 1 $
    \EndFor

\EndFor

\State \textbf{Return :} $\Sigma^T$
\end{algorithmic}
\end{algorithm}

\section{Impact of Pairwise Relationship in Reconstruction}
\label{app:locrec}

In this section, we report our findings on how the relationship between the attacker and its target affects the probability of reconstruction in gossip averaging. For a given attacker and target, the reconstruction can either be successful or failed, so it can be seen as a binary label. So how well can we predict this label based on the relation in the graph? This can be measured using Kendall rank correlation coefficient. We consider two relation metrics: the length of the shortest path between the attacker and the target, and their communicability \citep{Estrada_2008}, a measure used in graph mining. It is defined as a weighted average of the probability of going from one node to another in a given number of steps and measures how easy it is to communicate from one node to the other.

We use the same graphs as in the experiment on centrality, namely random Erd\"os Rényi graphs with $0$ as the attacker and the Facebook Ego graph $414$ where each node plays the role of the attacker in turn.
We report the results in \cref{tab:com}. We see that in both cases, the shortest path length provides a good insight on the reconstruction probability, whereas the communicability only shows a small correlation.

\begin{table}[h]
\centering
\begin{tabular}{@{}l@{\hspace{10pt}}c@{\hspace{20pt}}r@{\hspace{10pt}}c@{}}\toprule
Relationship & Erd\"os-Rényi graph & Facebook Ego graph \\
\midrule
Shortest Path Length & $-0.44 \pm 0.09$ &  $-0.51 \pm 0.13$ \\
Communicability & $0.35 \pm 0.05$ & $0.08 \pm 0.35$  \\
\bottomrule
\end{tabular}
\caption{Kendall rank correlation coefficient between communicability, shortest paths, and the probability of reconstruction.}
\label{tab:com}
\end{table}

\section{Example of reconstruction on Random Geometric graphs}

To illustrate the speed of reconstruction in gossip averaging, we report in \cref{fig:random} the reconstruction after $1$ iterations, $4$ iterations and $8$ iterations (doing more iterations does not allow to reconstruct more nodes). Note that our attack support non-fully connected graph, although nodes that are part of different components are of course not reconstructed.

\begin{figure}[t]
    \centering
    \subfigure[After $1$ iteration]{
        \includegraphics[width=0.3\textwidth]{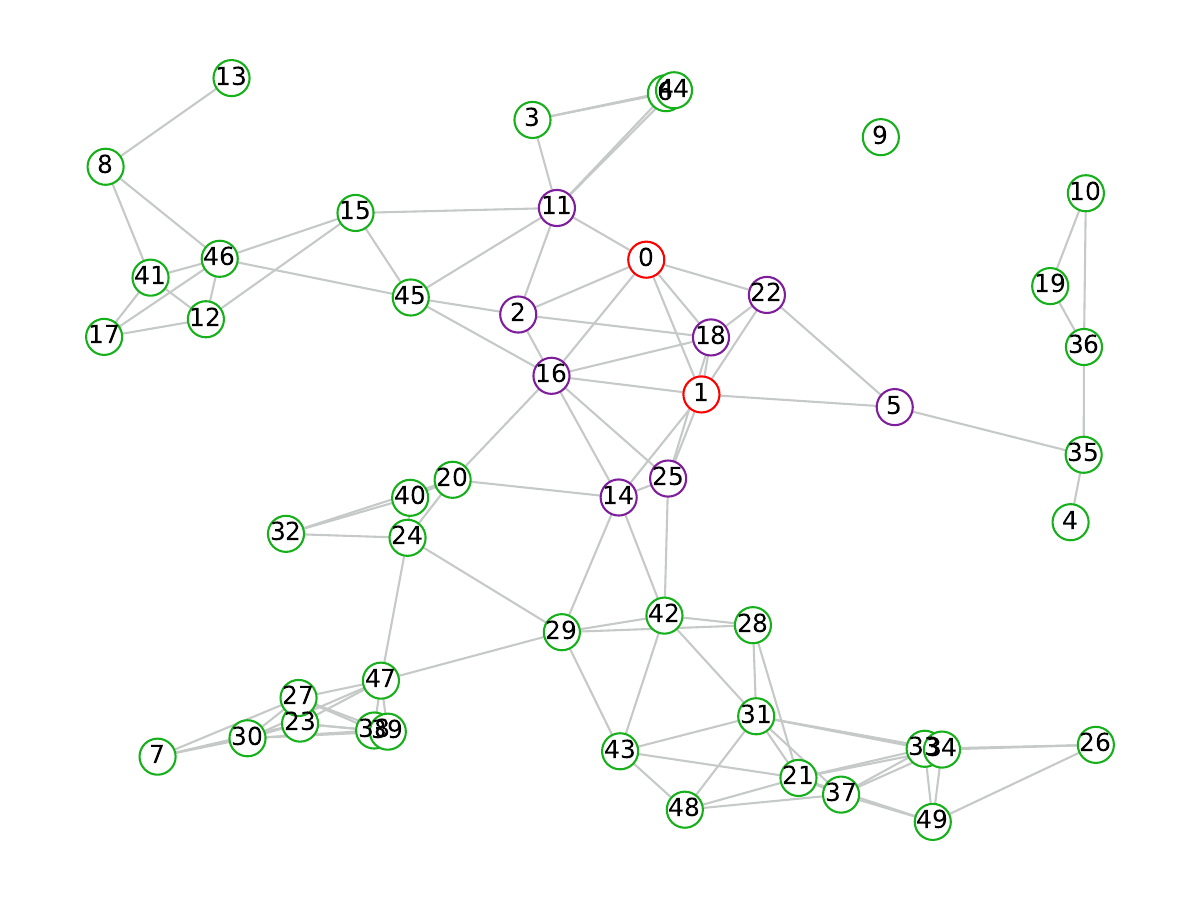}
        \label{fig:image1}
    }
    \subfigure[After $4$ iterations]{
        \includegraphics[width=0.3\textwidth]{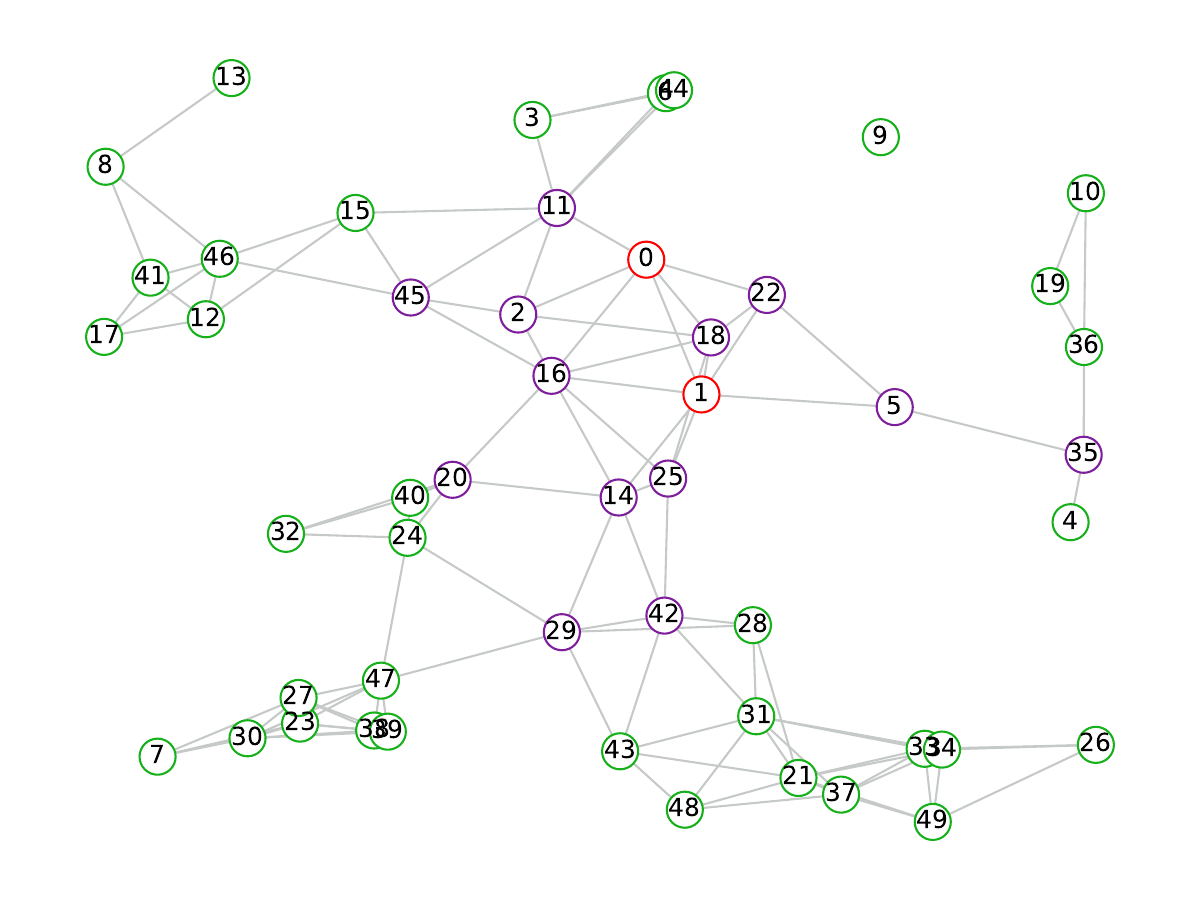}
        \label{fig:image2}
    }
    \subfigure[After $8$ iterations]{
        \includegraphics[width=0.3\textwidth]{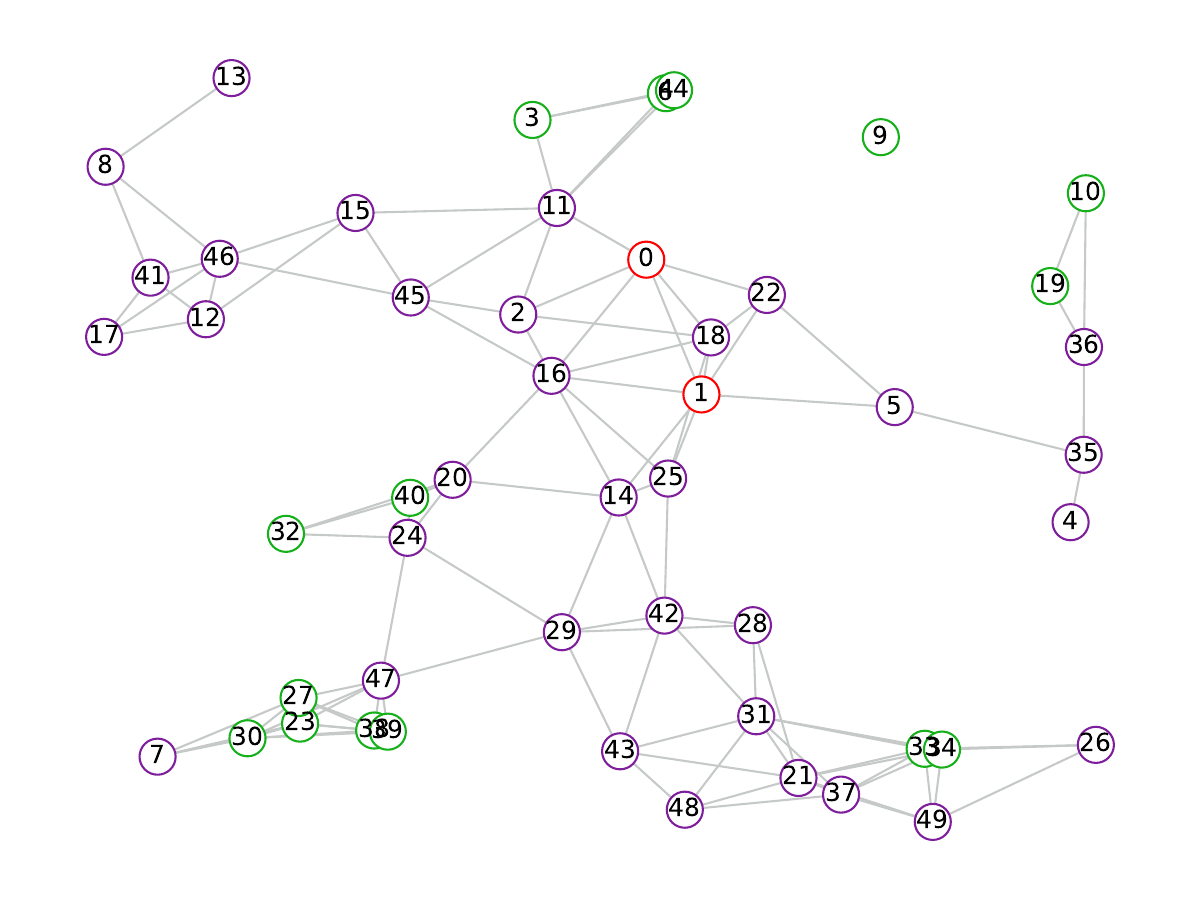}
        \label{fig:image3}
    }
    \caption{Reconstruction after different number of steps of gossip averaging, on a random geometric graph of $50$ nodes uniformly drawn from the unit square and a radius of $0.2$. Attackers are in red, reconstructed nodes in purple, and non-reconstructed ones in green.}
    \label{fig:random}
\end{figure}

\begin{figure}[t]
    \centering
        \centering
        \includegraphics[width=.4\textwidth]{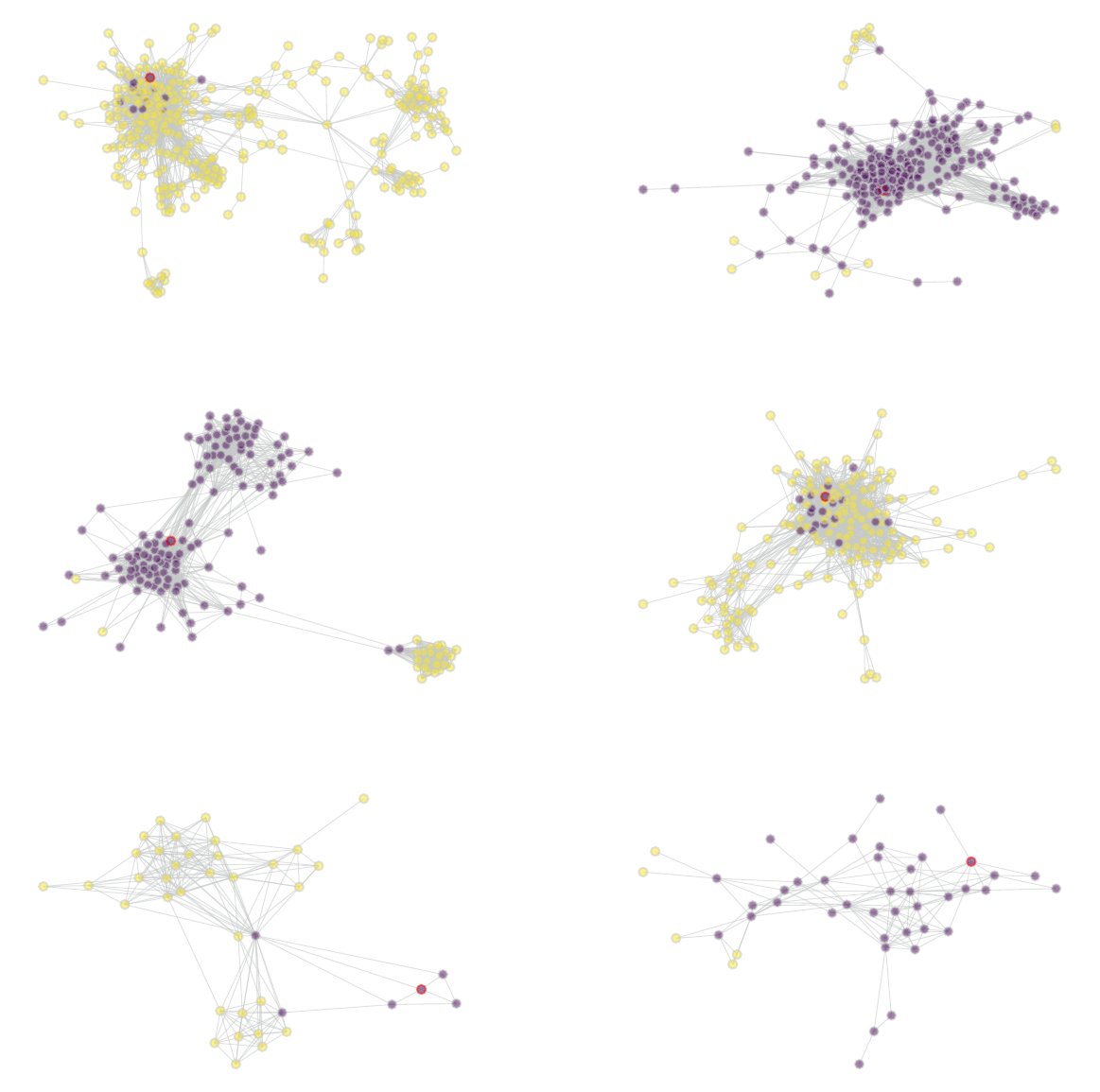}
        \includegraphics[width=.4\textwidth]{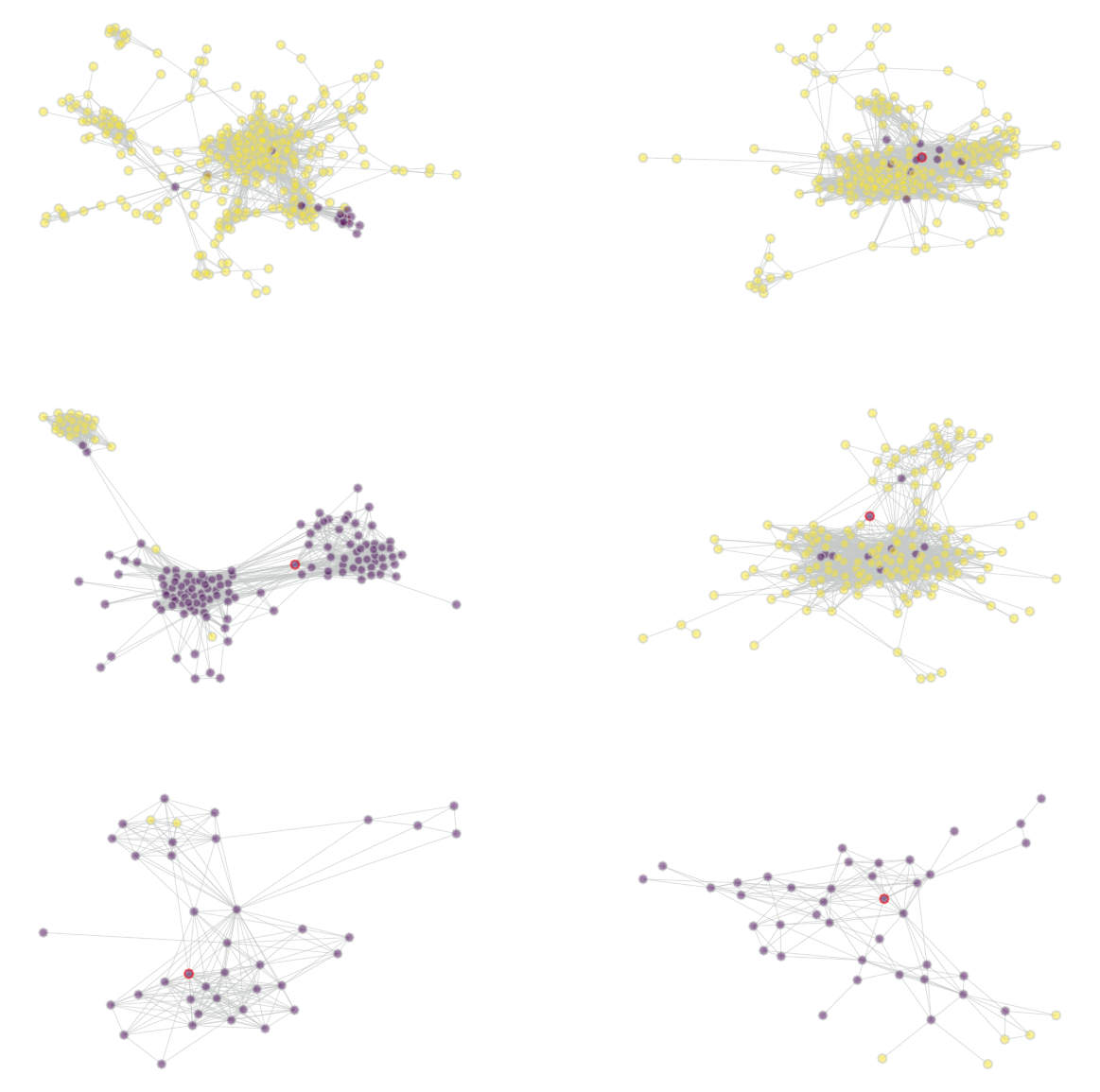}
    \caption{Reconstruction attack on gossip averaging for several Facebook Ego graphs with different attackers chosen randomly. The node circled in red is the attacker, with reconstructed nodes shown in purple and non-reconstructed ones shown in yellow.}
    \label{fig:allego2}
\end{figure}

\section{Reconstruction Attacks on Facebook Ego Graphs}
\label{app:ego}

Figure~\ref{fig:allego2} shows the results of our reconstruction attack on gossip averaging for several Facebook Ego graphs. We report the same graphs twice but with different choices of attackers. This illustrates that the proportion of reconstructed nodes depends both on the graph structure and the specific choice of the attacker. Note that in most cases, a vast majority of the nodes of the same community see their data leaked.

\section{Additional Experimental Results for Reconstruction Attacks on D-GD}
\label{app:expe_dgd}

In Figure~\ref{fig:ReconstructionExampleFlorentine}, we show the true inputs alongside the reconstructed images for the reconstruction attack shown in Figure~\ref{fig:ReconstructionFlorentine}.

In Figure~\ref{fig:ReconstructionExampleFlorentineMnist}, we show an example of reconstruction on the Florentine graph with the MNIST dataset using a Convolutional Neural Network (details in Table~\ref{table:model_architecture}).

\begin{figure*}
    \centering
    \includegraphics[width=\textwidth]{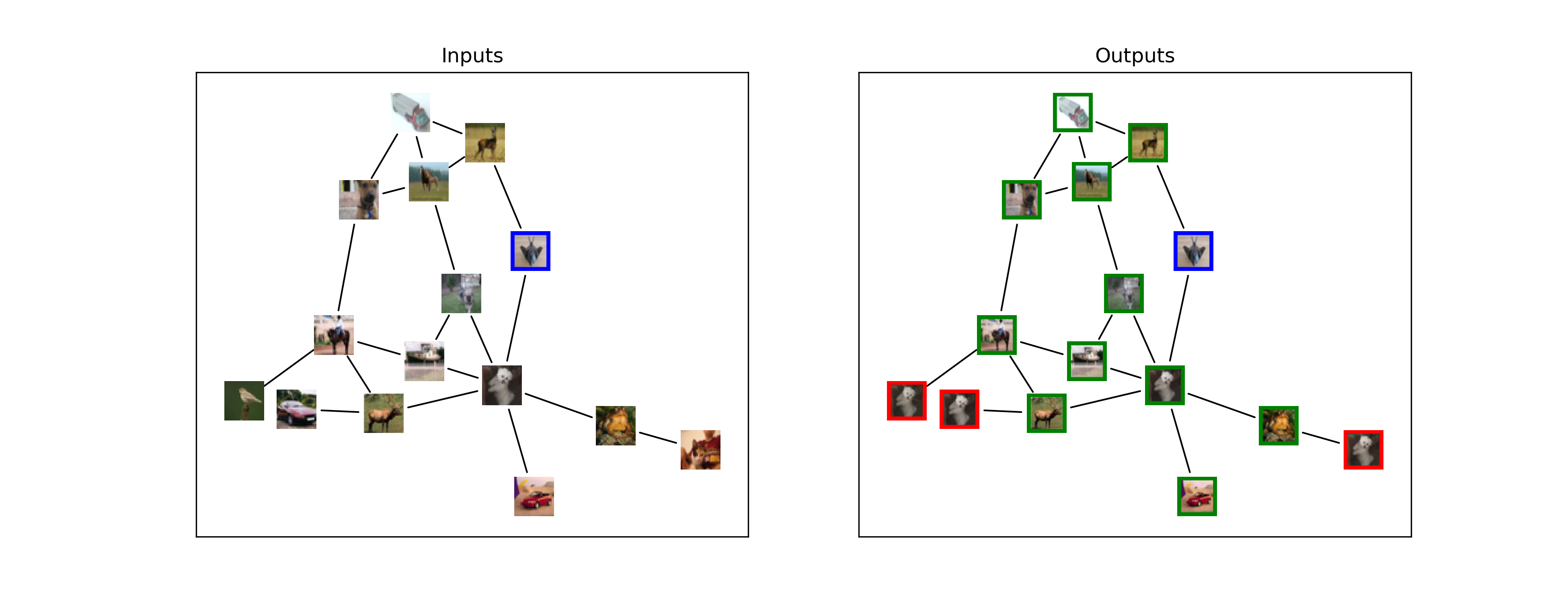}
    \caption{Reconstruction attack on D-GD for the Florentine graph showing the true image inputs (left) and the reconstructions (right).  The attacker is the node with the blue borders. Nodes with green borders are accurately reconstructed, the ones with red borders are not.}
  \label{fig:ReconstructionExampleFlorentine}
\end{figure*}

\begin{figure*}
    \centering
    \includegraphics[width=0.9\textwidth]{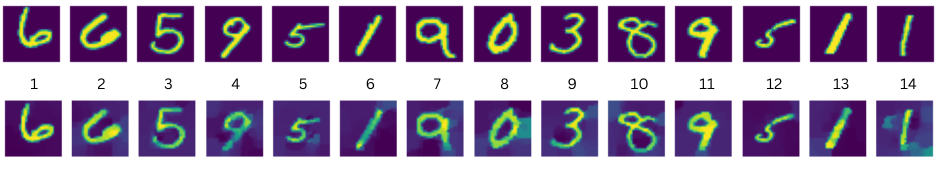}
    \caption{Reconstruction results on D-GD for the Florentine graph. The first (resp. second) row shows the true (resp. reconstructed) inputs (Learning rate $10^{-5}$ and CNN model from \cref{table:model_architecture}). The indices refer to the node labeling in Figure~\ref{fig:ReconstructionFlorentine} where the attacker is node $0$. }
  \label{fig:ReconstructionExampleFlorentineMnist}
\end{figure*}

\section{Details about the Convolutional Network}
\label{app:details}

We report in \cref{table:model_architecture} the following Convolutional Neural Network for the experiments with MNIST.

\begin{table}[h]
\centering
\begin{tabular}{@{}lll@{}}
\toprule
\textbf{Layer Type}    & \textbf{PyTorch Notation Size}  & \textbf{Activation/Pooling} \\ 
\midrule
Convolution            & (1, 32)                & ReLU, Max Pooling           \\ 
Convolution            & (32, 64)               & ReLU, Max Pooling           \\ 
Fully Connected        & (4096, 1024)           & ReLU                        \\ 
Fully Connected        & (1024, 10)             & -                           \\ 
\bottomrule
\end{tabular}
\caption{Model Architecture Description}
\label{table:model_architecture}
\end{table}

\section{Influence of the Learning Rate on the Attack on D-GD}
\label{app:learning_rate}

In \cref{fig:learning_rate}, we show the influence of the learning rate. When the learning rate becomes too large, gradients vary too much across iterations and it becomes impossible to make accurate reconstructions.

\begin{figure}
    \centering
    \includegraphics[width=0.55\textwidth]{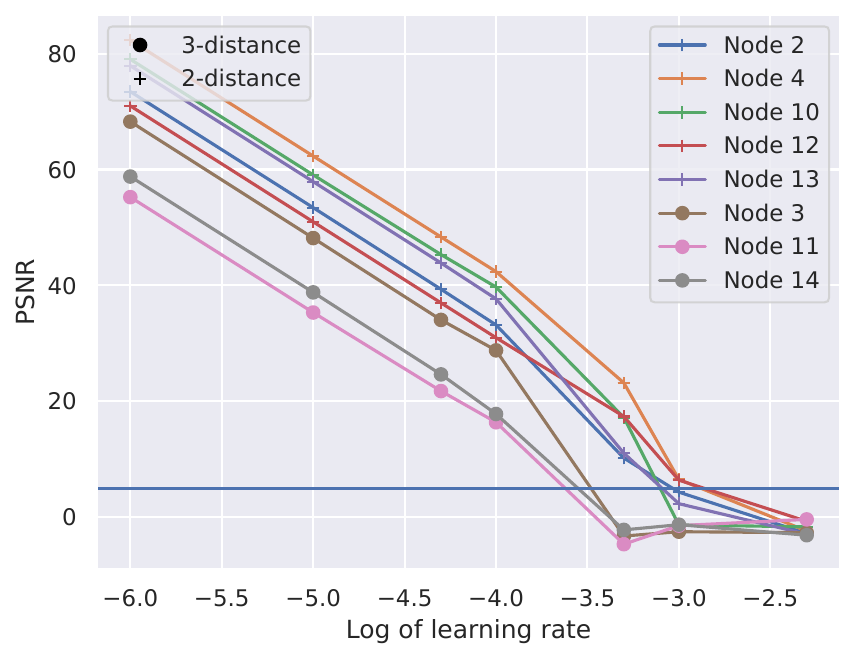}
    \caption{Impact of the learning rate on the reconstruction attack on D-GD for the Florentine graph experiment shown in  \cref{fig:ReconstructionFlorentine}. We plot the PSNR (averaged over 4 runs) between the reconstructed and original images for each of the target nodes and for different learning rates. We use different markers to classify the nodes that are at a distance of $2$ from the attackers and the ones that are at a distance of $3$ (The distance here refers to that of the shortest path between the attacker and the target). Points above the blue horizontal line are reconstructed with good visual quality. }
    \label{fig:learning_rate}
\end{figure}

\section{Discussion of the Assumption of Public Knowledge of the Gossip Matrix}
\label{app:pubW}

In our work, we assume that the attackers know the graph and the gossip matrix. While these quantities may not be fully known in some use-cases, we believe this is a justified assumption for the following reasons:

\begin{enumerate}
    \item In many real-world scenarios, the graph topology is, or can be extracted from, public information. This is the case when nodes are hospitals in various universities (these collaborations are typically public knowledge) or financial institutions (e.g., SEC requires financial ties to be made public), in the context of computations over social network graphs such as Mastodon, and when learning within blockchain networks or distributed ledgers.

\item In general, it appears to be unsafe to assume that the graph and gossip matrix $W$ can be kept fully hidden from the attacker nodes. Since they are part of the learning process, attacker nodes must know at least their corresponding row. From this knowledge, with enough attacker nodes, they may be able to know/infer a large part of the graph. In particular, they can exploit the fact that $W$ must be doubly stochastic. Furthermore, keeping a graph private while publishing basic statistics (such as the spectral gap or the number of edges, triangles, or degree distribution) is also known to be hard. For instance, \citet[][Section II.C therein]{bo2021} gives an example where a node can infer the edges of the graph from its own edges and the spectral gap; yet the spectral gap is often used to determine how many gossip steps are needed to reach a given precision. Therefore, given the challenge of precisely quantifying the risk of graph reconstruction, we find it reasonable to assume both the graph and matrix $W$ to be public information. This stance aligns with the established literature on differential privacy in decentralized learning, where it is commonly assumed that adversaries have access to the graph and matrix $W$ \citep[see e.g.][and references therein]{cyffers2022muffliato}.
\end{enumerate}

\end{document}